\newif\ifICML
\ICMLfalse

\ifICML

\documentclass{article}

\usepackage{microtype}
\usepackage{graphicx}
\usepackage{subfigure}
\usepackage{booktabs} 
\usepackage{makecell}

\usepackage{hyperref}

\usepackage{icml2021}

\else
	\documentclass[a4paper,11pt]{article}
\usepackage[colorlinks,linkcolor=blue,citecolor=blue,urlcolor=magenta,linktocpage,plainpages=false]{hyperref}
	\usepackage{fullpage}
\usepackage{makecell}

\fi

\usepackage{amsmath, amsthm, amssymb}
\usepackage{bm}
\usepackage{graphicx}
\usepackage{xcolor}
\usepackage{xspace}
\usepackage{url}
\usepackage{multirow}

\ifICML

\else
	\usepackage{algorithm}
	\usepackage{algorithmic}
	
\fi

\newtheorem{lemma}{Lemma}
\newtheorem{thm}{Theorem}

\newtheorem{prop}{Proposition}
\newtheorem{fact}{Fact}

\newtheorem{claim}{Claim}

\newcommand{\remove}[1]{}

\newcommand{\mom}[1]{{\left\vert\kern-0.25ex\left\vert\kern-0.25ex\left\vert #1 \right\vert\kern-0.25ex\right\vert\kern-0.25ex\right\vert}}

\newcommand{\pvec}{\mathbf{p}}
\newcommand{\qvec}{\mathbf{q}}

\newcommand{\xvec}{\mathbf{x}}
\newcommand{\yvec}{\mathbf{y}}
\newcommand{\cvec}{\mathbf{c}}

\newcounter{mynotes}
\ifICML
	\icmltitlerunning{On the price of  explainability for some clustering problems}
\else
	\title{On the price of  explainability for some clustering problems}
	\date{}
\fi

\begin{document}

\author{  Eduardo Sany Laber  \\ PUC-Rio, Brazil \\ {\tt laber@inf.puc-rio.br} \and
Lucas Murtinho  \\PUC-Rio, Brazil  \\ {\tt lucas.murtinho@gmail.com}}


\ifICML
\else
	\maketitle
\fi


\begin{abstract}

\remove{Machine learning models and algorithms are used in a number of  systems that affect our daily life.
Thus, in some settings,  methods that are easy to explain or interpret may be highly desirable. 
The price of explainability can be thought of as the loss in terms of quality that is unavoidable if
if these systems are required to use  explainable methods.
}

The price of explainability for a clustering task can be defined  as the unavoidable loss, in terms of the 
objective function,
 if we force  the  final partition to be explainable.

Here, we study this price  for the following 
clustering problems: $k$-means, $k$-medians, $k$-centers and  maximum-spacing.
We provide upper and lower bounds for a  natural model where explainability is achieved via decision trees.  
For the $k$-means and $k$-medians problems our upper bounds improve those obtained by 
[Moshkovitz et. al, ICML 20] for low dimensions.

Another contribution is a simple and efficient algorithm for building explainable clusterings for the $k$-means problem. We provide empirical evidence that its performance is better than the current state of the art for decision-tree based explainable clustering.


\end{abstract}


\section{Introduction}
Machine learning models and algorithms have been used in a number of systems that take decisions that affect our lives. Thus, explainable methods are desirable so that people are able to have a better understanding of their behavior, which allows for comfortable use of these systems or, eventually, the questioning of their applicability. 

Although most of the work on the field of explainable machine learning has been focusing on supervised learning  \cite{ribeiro2016should, lundberg2017unified,DBLP:conf/icml/VidalS20}, there has recently been some effort to devise explainable methods for unsupervised learning tasks, in particular, for clustering \cite{dasgupta2020explainable,bertsimas2020interpretable}. We investigate the framework discussed  by \cite{dasgupta2020explainable}, where  an explainable clustering 
is given by a partition,  induced by the leaves of a decision tree, that optimizes
some predefined objective function.

 
Figure \ref{fig:exp-clustering}
shows a clustering with three groups induced by a decision tree with $3$ leaves.
As an example, the blue cluster can be explained as the set of points  that satisfy
{\tt Feature 1} $\le 70$ and {\tt Feature 2} $> 40$.
Simple explanations as this one are usually not available
for the partitions produced by popular methods such as the Lloyd's algorithm for the $k$-means problem.

\begin{center}
		\hspace{40pt} \includegraphics[scale=0.5]{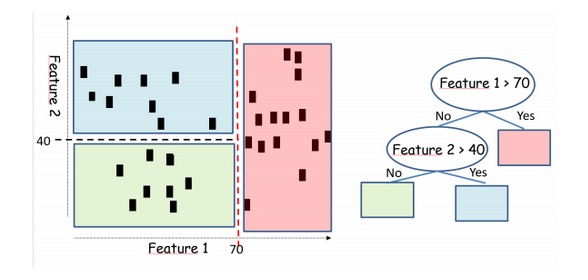}
\label{fig:exp-clustering}
	\end{center}

In order to achieve explainability,  one may be forced to accept some  loss in terms of the quality of
the chosen objective function (e.g. sum of squared distances).
In this sense, explainability has its price. 
  \cite{dasgupta2020explainable} presents theoretical  bounds on this price  for the $k$-medians and the $k$-means objective functions.

Here, we expand on their work by presenting  new bounds for these
objectives and also providing nearly tight bounds  for two other goals
that arise in relevant clustering problems, namely, the $k$-centers and the maximum-spacing problems.
We note that the objective for the latter is the one optimized by the widely known {\tt Single-Linkage} method,
employed for hierarchical clustering.  
We also give a more practice-oriented contribution
by devising  and evaluating a simple and efficient  algorithm for
building explainable clusterings for
the $k$-means problem.

\remove{
\cite{dasgupta2020explainable} present a way of approximating solutions for the $k$-means and $k$-medians clustering problems using small decision trees -- i.e., allowing only cuts that are perpendicular to the planes, and limited by previous cuts, when defining the clusters' borders. They show that ``any tree-induced clustering must in general incur an $\Omega(\log k)$ approximation factor compared to the optimal clustering'', as well as ``an efficient algorithm that produces explainable clusters using $k$ leaves'' with a constant factor approximation for $k = 2$ and, for general $k \geq 2$, 
``an $O(k)$ approximation to the optimal k-medians and an
$O(k^2)$ approximation to the optimal k-means''. Here we expand on their work by analyzing the approximation of $k$-centers clustering using decision trees.}

\subsection{Problem definition}
Let ${\cal X} $ be a set of $n$ points in $ \mathbb{R}^d$.
We say that a decision tree is {\em standard} if each internal node $v$ is associated with a
test (cut), specified by a coordinate $i_v \in [d]$ and a real value $\theta_v$, that partitions the  points
in ${\cal X}$ that reach $v$ into two sets:
those  having the  coordinate  $i_v$  smaller than or equal to $\theta_v$ and
those having  it larger than $\theta_v$. The leaves of a standard
decision tree   induce a partition of  $ \mathbb{R}^d$ into axis-aligned boxes and,
naturally, a partition of ${\cal X} $ into  clusters.

Let  $k \ge 2$ be an integer.   The clustering problems considered here consist of finding a  partition
of ${\cal X}$ into $k$ groups, among those that can be induced by a standard decision tree with $k$ leaves,
 that optimizes a given objective function.
For $k$-means, $k$-medians and $k$-centers, in addition to the partition, 
a representative $\mu(C) \in \mathbb{R}^d$ for each group $C$ must also
be output.

For the $k$-means problem the objective (cost function) to be minimized is the Sum of the Squared Euclidean Distances (SSED) between each point $\xvec \in {\cal X}$ and the representative of the cluster
where $\xvec$ lies. Mathematically, the cost (SSED) of a partition 
${\cal C}=(C_1,\ldots,C_k)$ for ${\cal X}$ is given by
$$cost({\cal C}) =\sum_{i=1}^k \sum_{\xvec \in C_i} || \xvec-\mu(C_i) ||_2^2.$$




The $k$-medians and the $k$-centers problems are  also minimization problems.
For the former, the cost of a partition ${\cal C}=(C_1,\ldots,C_k)$ is given by
$$cost({\cal C})= \sum_{i=1}^k \sum_{\xvec \in C_i} || \xvec-\mu(C_i) ||_1, $$
while  for the latter it is given by 
$$cost({\cal C}) =\max_{i=1,\ldots,k}  \max_{\xvec \in C_i} \{ || \xvec-\mu(C_i) ||_2 \} .$$




The maximum-spacing  problem is a maximization problem for which the objective to be maximized is the spacing
$sp({\cal C})$ of a partition  ${\cal C}$, 
 defined as 
$$ sp({\cal C})= \min \{ ||\xvec-\yvec||_2 : \xvec \mbox{ and } \yvec \mbox{ lie in distinct groups of ${\cal C}$}
 \}$$


We note that an optimal solution of the unrestricted version
 of any of these problems, in which the decision tree constraint is not enforced, might be a partition that is hard to explain in terms of the input features.
Thus, the motivation for using decision trees.

\remove{\cite{dasgupta2020explainable} informally define the price of explainability as "the multiplicative blowup in k-means (or k-medians) cost that is inevitable if we force our final clustering to have an interpretable form." 
A broader and more formal definition of the price of explainability
$\rho({\cal P})$ for  a clustering problem ${\cal P}$, with a minimization objective function,
 may be as follows: }

Along the lines of 
\cite{dasgupta2020explainable}, we define the price of explainability
$\rho({\cal P})$ for  a clustering problem ${\cal P}$, with a minimization objective function,
as
 $$\rho({\cal P})= \max_{I } \left \{ \frac{OPT_{exp}(I)}{OPT_{unr}(I)} \right \}, $$
where $I$ runs over all instances of ${\cal P}$;
 $ OPT_{exp}(I)$ is the cost of an optimal explainable clustering (via standard decision trees)   
for instance $I$ and  $OPT_{unr}(I)$  is the cost of an optimal 
unrestricted clustering  for $I$.
If ${\cal P}$ has a maximization objective function, then $\rho({\cal P})$
is defined as
$$ \rho({\cal P})=\max_{I } \left \{ \frac{OPT_{unr}(I)}{OPT_{exp}(I)} \right \}.$$

\remove{ 
  Let ${\cal P}$ be a clustering problem, with a minimization objective function, that falls into the discussed framework.
The price of explainability  for ${\cal P}$ in our setting  is  defined as
$$ \max_{I } \left \{ \frac{OPT_{exp}(I)}{OPT_{unr}(I)} \right \}, $$
where $I$ runs over all instances of ${\cal P}$;
 $ OPT_{exp}(I)$ is the cost of an optimal explainable (via decision trees) clustering  
for instance $I$ and  $OPT_{unr}(I)$  is the cost of an optimal 
unrestricted clustering  for $I$.
If ${\cal P}$ has a maximization objective function then the price of explainability 
is defined as
$$ \max_{I } \left \{ \frac{OPT_{unr}(I)}{OPT_{exp}(I)} \right \}.$$
}

\subsection{Our Contributions}
We provide bounds on the  price of explainability  as a function of the parameters $k,d$ and $n$
for the aforementioned  objective functions.
These objectives cover a spectrum
that includes both intra- and inter-clustering criteria as well as worst-case and average-case measures.

First, we address the $k$-centers problem. 
We show that
$$\rho(k\mbox{-centers}) \in \left \{
\begin{array}{ll}
\Omega (k^{1-1/d} ), &\mbox{ if } d \le \frac{ \ln k }{\ln \ln k}\\
\Omega \left ( \sqrt{d} \cdot \frac{k \cdot \sqrt{\ln \ln k} }{\ln ^{1.5} k} \right), & \mbox{ otherwise} 
\end{array}
\right. $$
and that
 $ \rho(k\mbox{-centers})$ is $ O( \sqrt{d} k^{1-1/d})$.
Our bounds are tight, up to constant factors, when $d$ is a constant.
For  an arbitrary $d$,  there is  only a polylogarithmic gap in $k$
between the upper and the lower bounds. The magnitude of this gap is exponentially smaller than that
of these bounds. 

\remove{
For the $k$-medians it is known that the price of explainability is
$O(k)$ and $\Omega( \log k)$ \cite{dasgupta2020explainable}. 
We contribute to the state of the art on this problem by showing that  $O(d \log k)$ is also an upper bound --
 an exponential improvement
for constant dimensions.
For the $k$-means problem, we also improve,
for low dimensions, the  $O(k^2)$ bound from \cite{dasgupta2020explainable} by proving 
that $\rho(k\mbox{-means})$ is   $O(k d \log k)$.
These upper bounds are obtained by exploiting an interesting connection with the literature of binary searching  
in the presence of non-uniform testing costs \cite{DBLP:journals/jcss/CharikarFGKRS02,DBLP:journals/siamcomp/LaberMP02}.
}

For the $k$-medians it is known that the price of explainability is
$O(k)$ and $\Omega( \log k)$ \cite{dasgupta2020explainable}. 
We contribute to the state of the art by showing that  $O(d \log k)$ is also an upper bound --
 an exponential improvement 
for constant dimensions. The upper bound follows from an interesting connection with the literature of binary searching  in the presence of non-uniform testing costs \cite{DBLP:journals/jcss/CharikarFGKRS02,DBLP:journals/siamcomp/LaberMP02}.

For the $k$-means problem, we also improve,
for low dimensions, the  $O(k^2)$ bound from \cite{dasgupta2020explainable} since we prove 
that $\rho(k\mbox{-means})$ is   $O(k d \log k)$.  Still, for the $k$-means problem,  we also  give a more practice-oriented contribution
by devising  and evaluating a simple and efficient greedy algorithm. Our method outperformed the 
{\tt IMM} method from \cite{dasgupta2020explainable} on an empirical study involving 10 real datasets. It should
be noticed that {\tt IMM} is a strong baseline since it got the best  results against 5 other competitors on the same datasets
according to \cite{dasguptaexplainable-workshop,frost2020exkmc}.

\remove{Due to the popularity of  the $k$-means problem, we also concentrate efforts on devising a practical algorithm. In fact, we propose a simple and efficient greedy algorithm that outperformed the 
{\tt IMM}  \cite{dasgupta2020explainable} on a set of  10 real datasets. It should
be noticed that {\tt IMM} obtained the best  results against 5 other baselines on these datasets
according to 
the empirical evaluation described in \cite{dasguptaexplainable-workshop,frost2020exkmc}.
 }

Finally, for  maximum-spacing we provide a tight bound by showing that the price of explainability
is $\Theta(n-k)$. The lower bound is particularly interesting since it shows that 
this objective function is bad for guiding explainable clustering, losing much more than  the
other considered objectives in the worst-case.



To derive our upper bounds, we analyze polynomial-time algorithms that start with an optimal 
$k$-clustering and transform it into an
 explainable one. The unrestricted versions of all the problems considered here, except for the  maximum-spacing problem, are {\cal NP}-Hard \cite{DBLP:journals/siamcomp/MegiddoS84,DBLP:journals/ml/AloiseDHP09}. However, all of them admit polynomial-time
algorithms with constant approximation \cite{DBLP:books/daglib/0030297,DBLP:journals/comgeo/KanungoMNPSW04}
and, hence, if we start with the partitions given by them, instead of the optimal ones, we obtain efficient
algorithms with provable approximation guarantees. These guarantees are exactly the upper bounds 
that we prove on the  price of explainability.




We believe that our results are helpful for the construction of explainable clustering solutions as well as for guiding the choice of an  objective function when explainability is required.


\remove{In this research we covered a range of objective functions that includes both intra and inner clustering measures as well as worst case and average case measures.  We believe that our results, as well as those with the same flavour,
are helpful to guide the choice of an  objective function when explainability is required
in clustering tasks.
 }


\remove{
The price of explainability in the clustering setting is defined as the maximum ratio
between the cost of the optimal unrestricted clustering and the explainable clustering.

Let  ${\cal I} (n,k,d) $ is  the set of all instances with $n$ points in the
$d$-dimensional space that need to be clustered in $k$ groups.
Formally, for a minimization objective function ${\cal O}$ (e.g. $\ell_1$ norm), the price of  explainability
 $PE_{\cal O}(n,k,d)$ for input parameters  $(n,k,d)$is given by 

$$ PE_{\cal O}(n,k,d) =\max_{I \in {\cal I} (n,k,d) } \left \{ \frac{  cost^{ex,{\cal O}} (I)  }{ cost^{un,{\cal O}} (I)}  \right \}, $$
where  
$cost^{ex,{\cal O}} (I) $ and $cost^{un,{\cal O}} (I) $  are, respectively,
the cost of the optimal explainable and the optimal unrestricted clustering for
instance $I$.

We study the price of explainability for the $k-$center objective  and  
the maximum spacing. We show that
}

\subsection{Related Work}


Our research is inspired by the recent work of \cite{dasgupta2020explainable},
where they propose an  algorithm, namely   {\tt IMM},
 for building explainable clusterings, via standard decision trees, for both the $k$-means and the $k$-medians problems. At each node  {\tt IMM} selects the cut that minimizes the number of points separated from their representatives
 in a reference  clustering. Our approach for these problems, while similar, uses a significantly different strategy to build the final decision tree, based on  trees that look at a single dimension of the data.
Moreover, as mentioned before, our algorithms provide better upper bounds for low dimensions.

\remove{
In \cite{dasgupta2020explainable}, an algorithm is presented that approximates a solution to the $k$-means or the $k$-medians problem through a decision tree that selects at each node the cut that minimizes the number of elements separated from their reference centers. Our approach for these problems, while similar, uses a significantly different strategy to build the final decision tree, based on decision trees that look at a single dimension of the data. The algorithm from \cite{dasgupta2020explainable} leads to a price of explainability that is $\Omega(\log k)$ and $O(k)$ for the  $k$-medians algorithm; our strategy is $O(d\log k)$, an improvement for low dimensions.
}

Decision trees have long been associated to hierarchical agglomerative clustering (HAC), which produces a hierarchy of clusters that is usually represented by a dendrogram. Examples of models that explicitly use decision trees for HAC include \cite{fisher1987knowledge, chavent1999methodes, blockeel2000top, basak2005interpretable}. To our knowledge, the use of decision trees for non-hierarchical clustering was first suggested in \cite{liu2000clustering}, in which a standard classification tree is used to identify dense and sparse regions of data. In \cite{fraiman2013interpretable}, unsupervised binary trees are also
 used to create interpretable clusters. More recently, an approach was presented in \cite{bertsimas2020interpretable} using optimal classification trees \cite{bertsimas2017optimal}, which are built in a single step by solving a mixed-integer optimization problem. For numerical databases, \cite{loyola2020explainable} presents a decision approach that decides on a split based on both the compactness of clusters and the separation between them.

The regions of space defined by decision-tree clustering will be hyper-rectangles (some of them may also be half-spaces if the overall region of interest is unbounded). Other approaches towards building hyper-rectangular clusters can be found in \cite{pelleg2001mixtures}, with a generative model, and \cite{chen2016interpretable}, with a discriminative one. Both models allow for probabilistic (soft) clustering, and \cite{chen2016interpretable} allows for incorporating previous knowledge to the model, but neither one guarantees that the resulting clusters can be represented by decision trees.

The main reason for using a (short) decision tree to build clusters is that the results of such algorithms are easily interpretable. Other avenues towards interpretable clustering have been explored in recent years. The technique presented in \cite{inconco} is based on the information-theoretic concept of minimum description length. In \cite{saisubramanian2020balancing}, a tunable parameter (the fraction of elements in a cluster that share the same feature value) leverages the tradeoff between clustering performance and interpretability. The same tradeoff is explored in \cite{frost2020exkmc} by relaxing the requirement from \cite{dasgupta2020explainable} that the explainable clustering should be induced by a tree with no more than $k$ leafs. In \cite{horel2020explainable}, a feature selection model from \cite{horel2019computationally} is used for clustering interpretation in the field of wealth management compliance. \cite{kauffmann2019clustering} uses a two-step approach, rewriting $k$-means clustering models as neural networks and applying to these networks techniques for interpreting supervised learning models. More information regarding explainable clustering may be found in \cite{chen2018interpretable, baralisexplainable}.

Of all the works mentioned in this section, only \cite{dasgupta2020explainable} presents  approximation guarantees  with respect to the optimal unrestricted (i.e., potentially uninterpretable) solution. Two algorithms from \cite{saisubramanian2020balancing} also have an approximation guarantee, but with respect to the optimal restricted (interpretable) solution, and the definition of interpretability in that work is quite different than ours (interpretable clusters are therein defined as those in which a given proportion of points share the same value for a predefined feature of interest).

Explainability and interpretability are topics of growing interest in the machine learning community \cite{ribeiro2016should, lundberg2017unified, adadi2018peeking, rudin2019stop,murdoch2019interpretable, molnar2020interpretable}. While there has been some focus on what \cite{dasgupta2020explainable} calls \textit{post-modeling explainability}, or the ability to explain the output of a black-box model \cite{ribeiro2016should, lundberg2017unified, kauffmann2019clustering}, the practice has also been criticized in contrast with \textit{pre-modelling explainability}, or the use of interpretable models to begin with \cite{rudin2019stop}. Our present work and \cite{dasgupta2020explainable} may be considered a middle-of-the-road approach, as the end result is a fully interpretable model (instead of, for instance, a model for locally interpreting the original model, or for explaining individual predictions) based on the output from a potentially black-box model.

\remove{
\subsection{Notation and Paper Organization}
For a positive integer $m$ we use $[m]$ to denote the set of
$m$ positive integers.

The paper is organized as follows...

}

\section{On the Price of Explainability for the $k$-centers problem}
In this section we address the $k$-centers problem.
We first present   a  lower bound  by 
constructing an instance  for which the price of  explainability is high.

\subsection{Lower Bound}

Let  $p \leq \min \{d,\log_3 k\}$ be a positive integer
whose exact value will be defined later in the analysis and  let $b$ be the largest integer for which $b^p \le k$.
Note that $b \ge 3$. Moreover,  let $k'=b^p$.

Our instance $I$ has $k + k' \cdot 2d$ points. We first 
 discuss how to construct the $k$ points, referred as centers, that will be set as representatives in an unrestricted 
 $k$-clustering for $I$ that has a low cost.
The first $k'$ centers will be obtained from the representation of the 
numbers $0,\ldots,k'-1$ in  base $b$ while the remaining $k-k'$ centers will be located sufficiently far from the
others so that they will be isolated in the low-cost $k$-clustering for $I$. 
Let $\cvec^0, \ldots,\cvec^{k'-1}$ be  the first $k'$ centers.

For a number
$i \in [k'-1]$ let $(i_{p-1},\ldots,i_0)_b$ be its representation in base $b$.
For $j \in [d]$, the value of the $j$-th  component of  center $\cvec^i$ is obtained by 
applying   $(j-1)$ times a circular shift on  
 $(i_{p-1},\ldots,i_0)_b$.
The  values of the remaining $d - p$ components of $\cvec^i$ are obtained by
copying the $p$ first values $d/p$ times so that
$c^i_j=c^i_{j'}$ if $(j-j') \mod p =0$.

As an example, if $b=3$, $p=3$ and $d=9$ then $\cvec^{14}=(14,22,16,14,22,16,14,22,16)$.
In fact, since $14=(1,1,2)_3$ we have that  $c^{14}_1=(1,1,2)_3=14$;
$c^{14}_2=(2,1,1)_3=22$ and 
$c^{14}_3=(1,2,1)_3=16$. The values of $c^{14}_4,\ldots,c^{14}_9$ are
obtained by repeating the first 3 values.

\remove{
 given 
by $$c^i_1= Dec((i_{p-1},\ldots,i_0)_b) = \sum_{\ell =0}^{p-1} i_{\ell} b^\ell.$$
The second component $c^i_2$ is given by $ Dec((i_0,i_{p-1},\ldots,i_1)_b)$,
the third one, $c^i_3$,  by $ Dec((i_1,i_0,i_{p-1},\ldots,i_2)_b)$ and so on.
}

The following observation is useful for our analysis.

\begin{fact} For every $\ell \in [p]$, the
values of the $\ell$-th coordinate of the $k'$ first centers are a  permutation of 
the integers $0,\ldots,k'-1$. 
\label{prop:all-integers}
\end{fact}

The remaining $k-k'$ centers, as mentioned above, should be far from each other and also far  away from
the $k'$ first centers. We can achieve that by setting  $\cvec^i=k^i \mathbf{1}$  for all
 $i>k'-1$, where $\mathbf{1}$ is the unit vector in $ \mathbb{R}^d$.

The next lemma gives a lower bound on the distance between any two centers.
\begin{lemma}
\label{lem:coordinate-wise}
For any two centers  $\cvec^i$ and $\cvec^j$, 
$$||\cvec^i-\cvec^j||_2 \ge \sqrt{\lfloor d/p \rfloor} \cdot (b^{p-1}/2 ).$$ 
\end{lemma}

\begin{proof}
If one of the two centers is not among the $k'$ first centers the result clearly holds.
Thus, we assume that $i,j \le k'-1$.

It is enough to show that   there is $\ell \in [p]$ for which  $|c^i_\ell-c^j_\ell| \ge  b^{p-1}/2$.
In fact, if this inequality holds for some $\ell$ then 
$|c^i_{\ell'}-c^j_{\ell'}| \ge  b^{p-1}/2$ for
each  $\ell'$ that is congruent to $\ell$ modulo $p$.
Since there are $\lfloor d/p \rfloor $ of them, due to our construction, we get  
the desired bound.

Let  $i=(i_{p-1},\ldots,i_0)_b$ and
$j=(j_{p-1},\ldots,j_0)_b$ be the 
representations of $i$ and $j$ in base $b$, respectively.
Let $f$  be such that $|i_f-j_f|$ is maximum.



Thus, the difference between $\cvec^i$ and $\cvec^j$ in the coordinate $[(f+1) \mod p ]+1$
is at least $$|i_f- j_f| \cdot \left ( b^{p-1} - \sum_{g=0}^{p-2} b^{g} \right )  \ge b^{p-1}/2,$$
where the last inequality holds because $|i_f- j_f| \ge 1$ and $b \ge 3$. \end{proof}

\remove{
{\it Case (ii).} $|i_f- j_f| = 1$.

We start with the very specific  situation where $d$ is even
and $i_g-j_g=1$ if $g$ is odd and $i_g-j_g=-1$ if $g$ is even.
In this case, $$c^i_1 - c^j_1 \ge b^{d-1} - b^{d-2} \ge b^{d-1}/2$
In the analogous situation where  
$i_g-j_g-1$ if $d$ is still even,  $g$ is odd and $i_g-j_g=1$ if $g$ is even,
we have $c^j_1 - c^i_1 \ge b^{d-1} - b^{d-2} \ge b^{d-1}/2.$

We observe if we are not in this specific situation then there must be
 an integer $g \in [d] $ for which one of the following conditions hold:
$i_g - j_g = 1$ and $i_{g+1} - j_{g+1} \ge  0$  or (b) $g=d$,
$i_d-j_d=1$ and $i_1-j_1 >0$.
In this case, in the dimension $g$ we have that
$$c^i_g - c^j_g \ge b^{d-1} - \sum_{g=0}^{d-3} b^{g} > b^{d-1}/2$$
}


Now, we define the remaining points of instance $I$.

For each of the first $k'$ centers we create $2d$ associated points: $\xvec^{i,1},\ldots,\xvec^{i,2d}$.
For $j=1,\ldots,d$, the  point $\xvec^{i,2j-1}$  is identical to $\cvec^i$ in all coordinates
but on the $j$-th one, in which  its value  is $c^i_j-3/4$. 
Similarly, the  point $\xvec^{i,2j}$ is identical to $\cvec^i$ in all coordinates
but in the $j$-th one, in which its value is $c^i_j+3/4$. 
By considering the $k$-clustering for $I$ where the
$k$ representatives are the $k$ centers $\cvec^{0},\ldots,\cvec^{k-1}$
and each point $\xvec^{i,j}$ lies in the group of $\cvec^{i}$,
we obtain the following proposition.

\begin{prop}
\label{prop:good-proposition}
There exists an unrestricted  $k$-clustering for instance $I$ with cost $3/4$.
\end{prop}

Now we analyse the cost of an optimal explainable clustering
for $I$. The following proposition 
is a simple consequence of Fact \ref{prop:all-integers}.

\begin{prop}
\label{prop:cut-separation}
Let $(j,\theta)$ be a cut that separates at least two  points from the 
set $A$ that includes the $k'$ first centers and its associated $k' \cdot 2d$ points.
Then, $(j,\theta)$ separates one point from its associated center.
\end{prop}

\ifICML

\else
\begin{proof}
Since $(j,\theta)$ separates at least two points from $A$ then
$\theta \in (-3/4,k'-1+3/4)$. 

If $\theta<0$, then  $(j,\theta)$ separates
the center that has the $j$-th coordinate equal to 0 from its associated point that has coordinate
$j$ equal to $-3/4$. If $\theta >k'-1$, then  $(j,\theta)$  separates
the center that has the $j$-th coordinate equal to 0 from its associated point that has coordinate
$j$ equal to $k'-1+3/4$.  Let $z$ be an integer that satisfies $0 \le z \le k'-2$
and such that  $\theta \in (z,z+1)$.  If
$\theta-z < 1/2$ (resp. $\theta-z > 1/2$),   $(j,\theta)$ separates the center that has the $j$-th coordinate equal
to $z$ (resp. $z+1$) from its associated point with $j$-th coordinate equal to $z+3/4$ (resp. $z+1 -3/4$).

Note that the existence of centers with the aforementioned values for coordinate $j$ is guaranteed
by Fact \ref{prop:all-integers}. 
 \end{proof}
\fi

\begin{lemma}
Any explainable $k$-clustering for instance $I$ has cost at least $\sqrt{\lfloor d/p \rfloor} \cdot ( b^{p-1}/4 )-3/8 $.
\label{lem:k-center-lower-bound}
\end{lemma}
\begin{proof}
Let ${\cal C}$ be an explainable $k$-clustering for instance $I$.
It is enough to show that there is a cluster $C \in {\cal C}$ that
contains two points, say $\xvec$ and $\yvec$, for which $$||\xvec-\yvec||_2 \ge \sqrt{\lfloor d/p \rfloor} \cdot ( b^{p-1}/2 )-3/4. $$
In fact,  in this case, due to the triangle inequality, for any choice of the representative for $C$, either
$\xvec$ or $\yvec$ will be at distance at least $\sqrt{\lfloor d/p \rfloor} \cdot ( b^{p-1}/4 )-3/8 $ from 
it.

If two centers lie in the same cluster of ${\cal C}$ then it follows from 
Lemma \ref{lem:coordinate-wise} that their distance
is at least  $\sqrt{\lfloor d/p \rfloor} \cdot ( b^{p-1}/2)$.


On the other hand, if every center lies on a different cluster in ${\cal C}$ then let $\xvec$ be the point that was separated
from its center, say $\cvec^i$, by a cut that satisfies the condition of Proposition 
\ref{prop:cut-separation}. 
Then, $\xvec$ lies in the same  cluster of $\cvec^j$, for some $j \ne i$.
From the triangle inequality we have that 
$$  ||\cvec^i -\cvec^j||_2  \le ||\cvec^i -\xvec||_2 +  ||\cvec^j -\xvec||_2.$$
Hence,  $||\cvec^j -\xvec||_2 \ge   \sqrt{\lfloor d/p \rfloor} \cdot ( b^{p-1}/2) -3/4$.
 \end{proof}

\remove{
\begin{thm}
The price of explainability for the $k$-centers problem is $\Omega (k^{1-1/d} )$ 
$d < \frac{ \log k }{\log \log k} $ and 
$$\Omega \left ( \sqrt{d} \cdot \frac{k \cdot \sqrt{\log \log k} }{\log^{1.5} k} \right), $$
if $d > \frac{ \log k }{\log \log k}$
\label{thm:lower-bound-kcenter}
\end{thm}
}

By putting together Proposition \ref{prop:good-proposition} and Lemma \ref{lem:k-center-lower-bound} and, then, optimizing
the value of $p$ we
obtain the  following theorem.

\begin{thm}
The price of explainability for the $k$-centers problem satisfies
$$\rho(k\mbox{-center}) \in \left \{
\begin{array}{ll}
\Omega (k^{1-1/d} ), &\mbox{ if } d \le \frac{ \ln k }{\ln \ln k}\\
\Omega \left ( \sqrt{d} \cdot \frac{k \cdot \sqrt{\ln \ln k} }{\ln^{1.5} k} \right), & \mbox{ otherwise.} 
\end{array}
\right. $$
\label{thm:lower-bound-kcenter}
\end{thm}

\ifICML
\else
\begin{proof}
Proposition \ref{prop:good-proposition} assures the existence of a $k$-clustering of   cost $3/4$ for instance $I$.
Let ${\cal C}$ be an  explainable clustering for $I$
and recall that $b^{p} = k'$.
It follows from the previous lemma that 
$$cost({\cal C}) \ge \sqrt{\frac{d}{p}} \cdot \frac{ b^{p-1}}{4} -3/8 = \sqrt{\frac{d}{p}} \cdot \frac{ (k')^{\frac{p-1}{p} } }{4} -3/8. $$

Since  $(b+1)^{p}>k$ we have
$$k' > \frac{ k}{ (1+1/b)^p} > \frac{ k}{ \exp(p/b)}   .$$
Thus,  
$$ cost({\cal C}) \ge  \sqrt{\frac{d}{p}} \cdot \frac{ k^{\frac{p-1}{p} } }{4\exp((p-1)/b)} -3/8. $$

Now we set $p=d$ if $d \le \frac{ \ln k }{\ln \ln k}$
and $p= \frac{ \ln k }{ \ln \ln k}$, otherwise.
Since $b>k^{1/p}-1$ we have that $b > \ln k  - 1 > p-1 $ for both cases and, hence,
$$ cost({\cal C}) \ge  \sqrt{\frac{d}{p}} \cdot \frac{ k^{\frac{p-1}{p} } }{4} -3/8. $$
By replacing $p$ in the previous equation according to each of the cases we obtain the desired result.
\end{proof}
\fi


\remove{
We obtained so far an instance for the case in which $d \le \lfloor \frac{\log k}{ \log \log k} \rfloor$.
If $d > \lfloor \frac{\log k}{ \log \log k} \rfloor$ we create an instance $I'$ in which the centers
are identical to centers of instance $I$ with respect to the first 
$ \lfloor \frac{\log k}{ \log \log k} \rfloor$ coordinates. 
The value of the $\ell$-th coordinate of a center $c^i$ in the new instance,
when $\ell > \lfloor \frac{\log k}{ \log \log k} \rfloor$ is the value of the
 $(\ell \mod \lfloor \frac{\log k}{ \log \log k} \rfloor) +1$ coordinate of center $c^i$ in the instance $I$.
Thus, the center $c^i$  of the new instance is obtained by making $d / \frac{\log k}{ \log \log k}$ copies of the values 
of the coordinates of $c^i$ in instance $I$.
The consequence is that ant two centers will differ by at least $\frac{k}{\log k} $ in at least
$d/ \lfloor \frac{\log k}{ \log \log k} \lfloor$ coordinates and, hence, the distance between any 2 centers
is at least 

$$ \sqrt{ \frac{d}{\lfloor \frac{\log k}{ \log \log k} \rfloor } }  \frac{k}{\log k}$$

The points associated with centers in the new instance are constructed exactly as in  instance 
$I$.  
}


\subsection{Upper bound}

In this section we show that the price
of explainability for the $k$-center problem is $O\left( \sqrt{d}k^{\frac{d-1}{d}} \right)$. 
Note that, for constant $d$, the upper bound matches the lower bound given by Theorem \ref{thm:lower-bound-kcenter}.

To obtain the upper bound we analyze the cost of the explainable clustering
induced by the decision tree built by the  algorithm presented in  
Algorithm \ref{alg:k-center}.

The algorithm has access to  the set of representatives of  an optimal
$k$-clustering ${\cal C}^*$ for ${\cal X}$. These representatives are used as {\em reference centers}
for the points in ${\cal X}$, that is, the reference center of a point $\xvec$ is the representative
of  $\xvec$'s group in ${\cal C}^*$.

Let ${\cal X}' $ and $S$ be, respectively, the subset of points in ${\cal X}$ and the set of  reference centers that reach a given node $u$.
To split $u$,  as long as  it is possible, the algorithm applies an axis-aligned
cut that does not separate any point $\xvec \in {\cal X}'$ from its reference center. 
This type of cut is referred as a {\em clean cut} with respect to  $({\cal X}',S)$.  
When there is no such  cut available for  $u$, the algorithm
partitions the bounding box of the points in ${\cal X}' \cup S$ into  $\lfloor |S|^{1/d} \rfloor^{d}$ axis-aligned boxes of the same dimensions
by using a decision tree that emulates a grid.
By the  bounding box of ${\cal X}' \cup S$ we mean the smallest box (hyper-rectangle) 
with axis-aligned sides that includes the points in ${\cal X}' \cup S$.

\begin{algorithm}[H]
  \caption{{\tt Ex-kCenter}( ${\cal X}'$: set of points)}

   \begin{algorithmic}[]
  	
  	\small
 	    \STATE  $S \leftarrow $ reference centers of the points in ${\cal X}'$
		\IF{$|S|=1$ }
          	\STATE Return ${\cal X}'$ and the single reference center in $S$ 
		\ELSE
        \IF{there exists a clean cut  w.r.t. $({\cal X}',S)$}
        
          \STATE $({\cal X}'_L,{\cal X}'_R)$ $\leftarrow $ partition  induced by the clean cut 
 
    	    \STATE Create a node $u$ 
   	      
    	      \STATE $u$.{\tt LeftChild} $\leftarrow ${\tt Ex-kCenter}(${\cal X}'_L)$ 
    	      
    	      \STATE $u$.{\tt RightChild} $\leftarrow ${\tt Ex-kCenter}(${\cal X}'_R)$ 

           \STATE Return the tree rooted at $u$
	    
	    \ELSE  
           \STATE $H \leftarrow $ bounding box for ${\cal X}' \cup S$ 
         

            \STATE $D^u \leftarrow $ decision tree that partitions $H$ into $ \lfloor |S|^{1/d} \rfloor^d$ identical axis-aligned boxes
             

            
            

           \STATE Return  $D^u$ as well as an arbitrarily chosen representative for each  of its leaves 
 
  
        \ENDIF
		
	   \ENDIF	

  \end{algorithmic}
  \label{alg:k-center}
\end{algorithm}

\begin{thm}
    The price of explainability for $k$-centers is  $O\left( \sqrt{d}k^{1-1/d} \right)$.
\end{thm}
\begin{proof}
We  argue that for each leaf $\ell$ of the tree ${\cal D}$  built by 
{\tt Ex-kCenter}(${\cal X}$),
the maximum distance between a point in $\ell$ and its representative  is
$ OPT  \sqrt{d}k^{1- 1/d}$,
where $OPT$ is the cost of the optimal unrestricted clustering.

We split the proof into two cases. The first case addresses the scenario in which
only clean cuts are used in the path from the root of ${\cal D}$ to the  leaf $\ell$.
The second case addresses the remaining scenarios.

\remove{
We split the proof into two cases. The first case addresses the scenario in which a leaf $\ell$
is  obtained  due to the condition $|S|=1$ (first condition of  the algorithm) while the second case handles
the scenario in which  $\ell$ is a  leaf in $D^u$.
}

\medskip

{\it Case 1.}
In this case all points that reach  $\ell$
lie in the same cluster of the optimal unrestricted $k$-clustering  ${\cal C}^*$.  Thus, the maximum distance from a point in $\ell$ to 
the single reference center in $S$ is upper bounded
by $OPT$.

\medskip

{\it Case 2.}
Let $u$ be the first node in the path from the root to $\ell$ for which
a clean cut is not available.   
Moreover,  let ${\cal X}^u$
be the set of points that reach $u$ and let $s=|S|$, that is, the  number of  reference centers that reach $u$.
In this case the algorithm splits the bounding box for ${\cal X}^u \cup S$ into
boxes 
of  dimensions $$\frac{L_1}{ \lfloor s^{1/d} \rfloor} \times \cdots \times \frac{L_d}{ \lfloor s^{1/d} \rfloor},$$
where $L_i$ is the difference between the maximum and minimum values of the $i$-th coordinate 
among points in ${\cal X}^u \cup S$. 

The maximum distance between a  point  in $\ell$ and its representative can be upper bounded by  the length of the diagonal
of the axis-aligned box corresponding to $\ell$. 
Let $m \in [d]$ be such that $L_m= \max \{L_1,\ldots,L_d\}$. Then, the length of the diagonal is upper bounded by 
$ L_m \sqrt{d} / \lfloor s^{1/d} \rfloor   \le  2 L_m  \sqrt{d} /s^{1/d}$.

Thus, it suffices to show that $OPT \ge L_m/(2s)$.
Let $\cvec^1,\ldots,\cvec^{s}$ be the  $s$ reference centers that reach node $u$.  
In addition, let
$\xvec^j$ be a point in ${\cal X}^u$  
with reference center $\cvec^j$ and such that 
$|x^j_m -c^j_m |$ is maximum, among the points  in ${\cal X}^u$ with
reference center  $\cvec^j$. Then, we must have
$$\sum_{j=1}^{s} 2 |x^j_m -c^j_m | \ge L_m, $$
for otherwise there would be a clean cut $(m,\theta)$, with 
$\theta \in [a,b]$,
where $a= \min \{y_m | \yvec \in {\cal X}^u \cup S \}$
and $b= \max \{y_m | \yvec \in {\cal X}^u \cup S \}$.
 Hence, for some point $\xvec^j$, 
$|x^j_m -c^j_m | \ge L_m / (2s)$. 
Since $OPT \ge |x^j_m -c^j_m | $ we get that 
$OPT \ge L_m/(2s)$.
\end{proof}

\section{Improved Bounds on $k$-medians for low dimensions}
We  show that the price of explainability for  $k$-medians is $O(d \log k )$,  which improves
the bound from \cite{dasgupta2020explainable} when $d= o(k/\log k)$.



\remove{
restrict our attention to  decision trees that
 take as input a set of $k$ reference centers and, then, use 
 at each of their nodes  axis-aligned cut that separate at least two reference centers.
 As a consequence, each of the $k$ clusters associated with the leaves of the  tree contains
exactly one reference center. 
}

As in the previous section we use  an optimal unrestricted  $k$-clustering 
${\cal C}^*$ for ${\cal X}$ as a guide for building an explainable clustering. Again, by the reference center of 
a point $\xvec \in {\cal X}$ we mean its representative in ${\cal C}^*$.

We  need  some  additional notation. For a
decision tree ${\cal D}$ and a node $u \in  {\cal D}$,
let $diam(u)$ be the $d$-dimensional vector whose
$i$-th coordinate $diam(u)_i$ is given by the difference between the maximum
and the minimum values of coordinate $i$ among the reference centers 
that reach $u$.  
Let $t_u$ be
the number of points that reach $u$ and  are separated from their
reference centers by the cut employed in $u$. Note that a point $\xvec\in {\cal X}$ can only
contribute to $t_u$  if both $\xvec$ and  its  reference
center reach $u$.
Finally,  we use $OPT$ to denote the cost of the optimal unrestricted clustering
${\cal C}^*$.

The following lemma from \cite{dasgupta2020explainable}, expressed in our notation,  
will be useful.

\begin{lemma}\cite{dasgupta2020explainable}
\label{lem:UB19Nov} Let ${\cal C}^*$ be an optimal unrestricted $k$-clustering for ${\cal X}$ and 
let ${\cal D}$ be a decision tree for ${\cal X}$ 
in which each representative of ${\cal C}^*$ lies in a distinct leaf.
Then, the clustering ${\cal C}$ induced by ${\cal D}$ satisfies  
\begin{equation}
\label{eq:23NovUB}
cost({\cal C}) \le OPT + \sum_{u \in {\cal D}} t_u ||diam(u)||_1.
\end{equation}
\end{lemma}

In order to obtain a low-cost explainable clustering we focus on 
 finding a decision tree ${\cal D}$ for which  
 the rightmost  term of the above inequality is small.
 This is the approach taken by {\tt IMM} \cite{dasgupta2020explainable}, 
a greedy strategy that at each node $u$ selects the cut
that yields the minimum possible value for $t_u$.

Although we follow the same approach, our strategy for building the tree  is significantly different.
In order to explain it,  we first  rewrite the rightmost term of (\ref{eq:23NovUB}):
\begin{equation} \sum_{u \in {\cal D}} t_u ||diam(u)||_1 =
 \sum_{i=1}^d \sum_{u \in {\cal D}}  t_u   diam(u)_i.
  \end{equation}

Motivated by Lemma \ref{lem:UB19Nov} and the above identity,
 our strategy  constructs $d$ decision trees ${\cal D}_1,\ldots,{\cal D}_d$,
where  ${\cal D}_i$ is built with the aim of minimizing  
\begin{equation}
\sum_{u \in {\cal D}} t_u diam(u)_i,
\label{eq:ub-21Jan}
\end{equation} 
ignoring the  impact
on the coordinates $j \ne i$.

Next,  it constructs  a decision tree ${\cal D}$ for ${\cal X}$
by picking nodes from these $d$ trees.
More precisely,  to split a node $u$ of ${\cal D}$ 
the strategy first selects a coordinate $i \in [d]$ for which $diam(u)_i$  is maximum.
Next, it 
applies the cut  that is associated with  the node in ${\cal D}_i$ which is the least common ancestor (LCA) of the 
set of reference centers that reach $u$. 

In the pseudo-code presented in Algorithm \ref{alg:randLearner}, $S'$ is a subset of
 the set $S$ of representatives of ${\cal C}^*$. Moreover, 
 ${\cal X}'$ is a subset of the points in ${\cal X}$. 
 The procedure is called, initially, with  ${\cal X}'= {\cal X}$ and $S'=S$.

\remove{
This is the same approach taken by \cite{} where
it is proposed  
a greedy strategy that at each node $u$ selects the cut
that yields the minimum possible value for $t_u$.

Our strategy, however, is significantly different. First it builds $d$ decision trees ${\cal D}_1,\ldots,{\cal D}_d$,
where  ${\cal D}_i$ is built with the aim of minimizing   $UB_i()$, ignoring its impact
on $UB_j()$ for $j \ne i$.
Then,  it constructs  a decision tree ${\cal D}$ for ${\cal X}$
by picking cuts associated with nodes from these $d$ trees.
In fact,  to split a node $u$ of ${\cal D}$ 
the strategy first selects a coordinate $i \in [d]$ according to some criterion  and then
apply the cut  that is given by the node in ${\cal D}_i$ which is the least common ancestor (LCA) of the 
 set of reference centers that reach $v$.

 }

\begin{algorithm}[H]
  \caption{{\tt BuildTree}(${\cal X}' \cup S'$)}
  \begin{algorithmic}[]
  	\small

        \STATE Create a node $u$ and associate it with ${\cal X}' \cup S'$
		
		\IF{$|S'|=1$}
		\STATE Return the leaf $u$
		\ELSE

		\STATE  Select $i \in [d]$  for which $diam(u)_i$ is maximum.
						
		\STATE  $v \leftarrow$  node in ${\cal D}_i$ which is the LCA of the centers in $S'$
		
		\STATE  Split ${\cal X}' \cup S' $ into ${\cal X}'_{L} \cup S'_L $ and ${\cal X}'_{R} \cup S'_R$ using the cut associated with $v$.
		
		\STATE  $u$.{\tt LeftChild} $\rightarrow$ {\tt BuildTree}(${\cal X}'_{L} \cup S'_L$) 
		
		\STATE  $u$.{\tt RightChild}   $\rightarrow$ {\tt BuildTree}(${\cal X}'_{R} \cup S'_R$)
		
		\STATE Return the decision tree rooted at $u$		
		
		\ENDIF

  \end{algorithmic}
  \label{alg:randLearner}
\end{algorithm}


To fully specify the algorithm we need to explain how the decision trees ${\cal D}_i$ are built.
Let $\cvec^1,\ldots,\cvec^k$ be the reference
centers sorted by coordinate $i$, that is,
$c^{j}_i < c^{j_+1}_i$ for $j=1,\ldots,k-1$.
Moreover, let $(i,\theta^j)$ be  the cut that separates the  points in ${\cal X}$ 
with  the  $i$-th coordinate smaller than or equal 
 to $\theta^j=(c^{j}_i+ c^{j+1}_i)/2$ from the remaining ones.

For $1 \le a \le b \le k$, let 
  ${\cal F}_{a,b}$ be the family of binary  decision trees with $(b-a)$ internal nodes and
 $b-a+1$ leaves
 defined as follows:  
\begin{itemize}
\item[(i)] if $a=b$, then ${\cal F}_{a,b}$ has a single tree and this tree contains only one node.
\item[(ii)] if $a<b$, then ${\cal F}_{a,b}$ consists of all the decision trees ${\cal D'}$ with the following structure:
the root of ${\cal D'}$ is identified by  a number $j \in \{a,\ldots,b-1\}$ and associated with the cut
$(i,\theta^j)$; one child of the root of ${\cal D'}$
is a tree in the family ${\cal F}_{a,j}$ while the other is a tree in ${\cal F}_{j+1,b}$.
\end{itemize}

\remove{
For $1 \le a < b \le k$, let $C_{a,b}=\{\cvec^a,\ldots,\cvec^b\}$ and 
 let ${\cal F}_{a,b}$ be the family of binary  decision trees with $(b-a)$ internal nodes  that satisfy the following
conditions:  
\begin{itemize}
\item[(i)] each internal node is identified by a distinct number in the set $ \{a,\ldots,b-1\}$, so that  the one
identified by $j$ is  associated with  cut $(i,\theta^j)$;
\item[(ii)]  for every internal node $v$, the cut associated with $v$  separates at least two 
reference centers 
in $C_{a,b}$ that reach $v$.
\end{itemize}
}

For our analysis, in the next sections, it will be convenient to view  ${\cal F}_{a,b}$ as the family
of binary search trees for the numbers in the set $\{a,\ldots,b-1\}$.

  
Let $T_j$ be the number of points in ${\cal X}$ that are
separated from their centers by  cut 
$(i,\theta^j)$. 
For every tree ${\cal D}' \in {\cal F}_{a,b}$  we define $UB_i({\cal D}')$
as
$$ UB_i({\cal D}') =\sum_{j=a}^{b-1} T_j \cdot diam(j)_i,   $$
where $diam(j)$ is the diameter of the node identified by $j$ in ${\cal D}'$.

The tree  ${\cal D}_i$ is, then, defined as 
$${\cal D}_i = \operatorname{argmin}  \{   UB_i({\cal D}') \mid {\cal D}' \in {\cal F}_{1,k}  \}.$$ 

The motivation for minimizing $UB_i()$ is that  for every tree ${\cal D}' \in {\cal F}_{1,k}$,
  $UB_i()$ is an upper bound on (\ref{eq:ub-21Jan}), that is,
  $$  \sum_{u \in {\cal D}'} t_u diam(u)_i \le \sum_{j=1}^{k-1} T_j \cdot diam(j)_i =UB_i({\cal D}').$$
To see that, let $j$ be the integer identified with the node 
$u \in {\cal D}'$.   By definition  $diam(u)_i=diam(j)_i$.
Moreover,  we have $t_u \le T_j $ because $t_u$ only accounts the points that are separated
from their reference centers among those that reach $u$, while
  $T_j$ accounts all the points in ${\cal X}$ regardless of whether they reach $u$ or not.

We discuss how to construct ${\cal D}_i$  efficiently. Let
 $OPT_{a,b} =  \min \{ UB_i({\cal D}') \mid {\cal D}' \in {\cal F}_{a,b}  \},
$ if $a < b$, and let $OPT_{a,b}=0$ if $a=b$.
Hence, $UB_i({\cal D}_i)=OPT_{1,k}$.
The following relation holds for all $ a < b$:

\begin{equation}
OPT_{a,b} = \min_{a \le j \le b-1} \left \{  T_j (c^b_i -c^a_i )  + OPT_{a,j}+OPT_{j+1,b} \right \}.
\label{eq:dynprog2}
\end{equation}

Thus, given a set of $k$ reference centers and the values $T_j$'s,
${\cal D}_i$ can be computed in $O(k^3)$ time  by  solving equation (\ref{eq:dynprog2}),
for $a=1$ and $b=k$,
via standard dynamic programming techniques.


\subsection{Approximation Analysis: Overview}

We prove that the cost of the clustering induced by ${\cal D}$ is  $O(d \log k)\cdot OPT$.
To reach this goal, we 
first  show that

\begin{equation}
\label{eq:23NovUB2}
UB_i({\cal D}_i) \le 2 \log k \left ( \sum_{j=1}^{k-1} (c_i^{j+1}-c^j_i)T_j \right).
\end{equation}

The proof of this bound relies
on  the fact that ${\cal D}_i$ can be seen
as a binary search tree with non-uniform probing costs. We use   properties of this kind of tree, in particular  the one proved in 
\cite{DBLP:journals/jcss/CharikarFGKRS02} about its competitive ratio.


Let $$OPT_i=\sum_{\xvec \in {\cal X}} |x_i- c(\xvec)_i|$$ be the contribution of coordinate $i$
to $OPT$, where $c(\xvec)$ is the reference center of $\xvec$. 
Our second step consists of showing that
\begin{equation}
\label{eq:nov19-3}
 \left ( \sum_{j=1}^{k-1} (c_i^{j+1}-c^j_i)T_j \right)/2
  \le  OPT_i.
 \end{equation}
Roughly speaking, the proof of this bound consists of projecting the points of ${\cal X}$ and
the reference centers onto the axis  $i$
and then counting the number of times the interval $[c_i^{j},c_i^{j+1}]$ appears
in the segments that connect points in ${\cal X}$  to their reference centers.
This is exactly the same line of reasoning employed  to prove  Lemma 6 from  the supplementary version of  \cite{dasgupta2020explainable}. 

At this point, from the two previous inequalities,    we obtain

\begin{equation}
\label{eq:dez5}
UB_i({\cal D}_i) \le 4 \log k \cdot OPT_i.
\end{equation}
Finally, we prove that  a factor of $d$ is incurred when we build the tree ${\cal D}$ from
the nodes of the trees ${\cal D}_1,\ldots,{\cal D}_d$:

\begin{equation}
\label{eq:nov19}
\sum_{v \in {\cal D}} t_v ||diam(v)||_1 \le d \sum_{i=1}^d UB_i ({\cal D}_i).
\end{equation}

From (\ref{eq:dez5}), (\ref{eq:nov19}) and  the identity  $OPT=\sum_{i=1}^d OPT_i$,
we obtain 
$$\sum_{v \in {\cal D}} t_v ||diam(v)||_1 \le  4 d \log k \cdot OPT.$$
This together with 
  Lemma \ref{lem:UB19Nov} allows us to  establish the main theorem of this section.

\begin{thm}
The price of explainability  for $k$-medians is $ O( d \log k )$.
\end{thm}

\remove{We need some additional notation.
For a set of points $Y$, with some of them being reference centers,
we use $t^Y_u$ to denote the number of 
points in $Y$ that are separated from their reference centers by the cut associated with $u$.
If   $Y$ is the set of points that reach  the node $u$ then
 we keep using $t_u$ rather than $t^Y_u$ for the sake of  simplicity.
}

\ifICML
\else

\subsection{Approximation Analysis: Proofs}
We start with the proof of inequality (\ref{eq:23NovUB2}).

\begin{lemma}
The tree ${\cal D}_i$ satisfies 
$$UB_i({\cal D}_i) \le 2 \log k \left ( \sum_{j=1}^{k-1} (c_i^{j+1}-c^j_i)T_j \right).$$
\label{lem:UB-k-medians-1}
\end{lemma}
\begin{proof}

Let ${\cal D}'$ be a tree in ${\cal F}_{1,k}$. By construction,   the set of centers that reach the node in $ {\cal D}'$
identified by $j$ is a contiguous
subsequence of $\cvec^1,\ldots,\cvec^k$. Let $r(j)$ and $s(j)$ be, respectively,
the first and the last indexes of the  centers of this subsequence. Thus, 

\begin{equation}
\label{eq:23Nov-expr}
  UB_i({\cal D}') = \sum_{j=1}^{k-1} T_{j} \cdot diam(j)_ i = \sum_{j=1}^{k-1} T_{j} \sum_{\ell=r(j)}^{s(j)-1} (c^{\ell+1}_i-c^{\ell}_i ).
 \end{equation}

We can show that the right-hand side of the above equation satisfies


\begin{equation}
\sum_{j=1}^{k-1} T_{j} \sum_{\ell=r(j)}^{s(j)-1} (c^{\ell+1}_i-c^{\ell}_i ) =
 \sum_{\ell=1}^{k-1}  (c^{\ell+1}_i -c^{\ell}_i) \cdot \sum_{j  \in An(\ell,{\cal D}')} T_j,
\label{eq:22Dec-expr}
\end{equation} 
where  $An(\ell,{\cal D}')$ is the set of  nodes that are ancestors (including $\ell$)
of the node identified by $\ell$ in ${\cal D}'$.

To see that, fix $j,\ell \in [k-1]$. The term 
$T_j  (c^{\ell+1}_i -c_i^\ell )$
contributes the left-hand side of (\ref{eq:22Dec-expr})   
if the centers  $\cvec^{\ell}$ and $\cvec^{\ell+1}$ reach the node $j$ in ${\cal D}'$.
This happens if and only if $j$ is an ancestor of the node identified by $\ell$ in
${\cal D}'$.

\remove{

\begin{equation}
\sum_{j=1}^{k-1} T_{j} \sum_{\ell=r(j)}^{s(j)-1} (c^{\ell+1}_i-c^{\ell}_i ) =
 \sum_{j=1}^{k-1}  (c^{j+1}_i -c^{j}_i) \cdot \sum_{u  \in An(j,{\cal D}')} T_{u},
\label{eq:22Dec-expr}
\end{equation} 
where 
$An(j,{\cal D}')$ is the set of  nodes that are ancestors (including $j$)
of the node identified by $j$ in ${\cal D}'$.

In fact, 
for every $j,\ell \in [k-1]$, the term $T_j (c^{\ell+1}_i -c_i^\ell )$ contributes to
the left-hand side of (\ref{eq:22Dec-expr})   
if and only if  both  centers $\cvec^{\ell}$ and $\cvec^{\ell+1}$ reach the node $j$.
This happens exactly for  the nodes $u$ that are ancestors of
the node identified by $j$ in ${\cal D}'$. 
}

Now, we use  Theorem 4.5 from \cite{DBLP:journals/jcss/CharikarFGKRS02}.
It states  that for any  vector  $(p_1,\ldots,p_k)$  of $k$ non-negative real numbers  there exists  a binary search  tree $B$ having $k$ nodes, with each of them associated with a number in $[k]$, that 
satisfies 
$$\sum_{j \in An(\ell,B) } p_j \le  (\log k +o(\log k)) p_\ell \le  2 \log k  \cdot p_\ell,$$
for every node $\ell$ of $B$.

Let ${\cal D}_c$ be a  tree obtained via the result of \cite{DBLP:journals/jcss/CharikarFGKRS02} 
for the vector $(T_1,\ldots,T_{k-1})$.
It satisfies
$$ \sum_{j \in An(\ell,{\cal D}_c) } T_j \le 2 \log (k-1) \cdot  T_\ell.$$
By using this inequality, (\ref{eq:23Nov-expr}) and (\ref{eq:22Dec-expr}), we get that
$$ UB_i({\cal D}_c) = \sum_{\ell=1}^{k-1}  (c^{\ell+1}_i -c^{\ell}_i)   \sum_{j \in An(\ell,{\cal D}_c) } T_j \le 
2 \log k \sum_{\ell=1}^{k-1}  (c^{\ell+1}_i -c^{\ell}_i) T_\ell.$$
The result follows because the minimality of 
${\cal D}_i $ guarantees that  
$ UB_i({\cal D}_i)  \le UB_i({\cal D}_c) $.
\end{proof}

Inequality (\ref{eq:nov19-3}) is formalized in the next lemma.
\begin{lemma} Let $OPT_i$ be the contribution of the coordinate $i$ for the cost of
an optimal unrestricted clustering ${\cal C}^*$. Then,
\begin{equation}  OPT_i= \sum_{ \xvec \in {\cal X} } |x_i - c(\xvec)_i| \geq \sum_{j=1}^{k-1} 
\frac{(c^{j+1}_i-c^j_i) T_j}{2},
\end{equation}
where $c(\xvec)$ is the  reference center of $\xvec$.
\label{lem:k-means-lower-bound}
\end{lemma}
\begin{proof}
Let $\cvec^1,\ldots,\cvec^k$ be the reference centers sorted by increasing order of coordinate $i$.
Recall that $\theta^j=(c^j_i+c^{j+1}_i)/2$. For every $\xvec \in {\cal X}$, let
  $Cut(\xvec)=\{ j| (i,\theta^j) \mbox{ separates }  \xvec \mbox{ from } c(\xvec) \}$.


\remove{

\medskip

 {\it Case 1}: $x_i < c(x)_i$.  Let $r$ be the smallest integer (it it exists) for which the  midpoint
of interval $I_r$ has coordinate larger than $x_i$. 
Moreover,  let $s$ be the rank of $c(\xvec)$ in the list of centers sorted by coordinate $i$, that is,
$c(\xvec)=\cvec^s$.   If $s > r$ we have that
 
$$|x_i- c(x)_i|= c(x)_i - x_i   \ge \sum_{\ell=r}^s ( c^{\ell+1}_i-c^{\ell}_i)/2. $$
}


Fix $\xvec \in {\cal X}$. 
If $j \in Cut(\xvec)$ then either  $[c^{j}_i,\theta^j]$ or 
$[\theta^j,c^{j+1}_i]$ is included in the real interval
with endpoints $x_i$ and $c(x)_i$. Thus, 
we have that
$$|x_i- c(x)_i| \ge 
\sum_{j \in Cut(\xvec) } ( c^{j+1}_i-c^{j}_i)/2  $$

\remove{
\medskip

{\it Case 2}: $x_i \ge c(x)_i$. We also have that
$$|x_i- c(x)_i| =  x_i - c(x)_i  \ge 
\sum_{j \in Cut(\xvec) } ( c^{j+1}_i-c^{j}_i)/2  $$
}

\remove{
{\it Case 2}: $x_i \ge c(x)_i$.
In this case, let
$c(\xvec)=\cvec^r$ and let $s$
be the largest integer (if it exists) for which the midpoint of $I_s$ has coordinate smaller than $x_i$.
If $s \ge r$, we have that
 
$$|x_i- c(x)_i| =  x_i - c(x)_i   \ge \sum_{\ell=r}^s ( c^{\ell+1}_i-c^{\ell}_i)/2. $$
}

By  adding the above inequality for all $\xvec \in {\cal X}$ we conclude
that the number of times  that 
$(c^{j+1}_i-c^j_i)/2$ contributes to the right-hand side,  for every $j \in [k-1]$, is
exactly the number of times that  $(i,\theta^j)$
separates a point $\xvec' \in {\cal X}$ from its reference center $c(\xvec')$. This number is exactly $T_j$.
\end{proof}

Finally, we present
the proof of inequality (\ref{eq:nov19}).

\begin{lemma}
Let ${\cal D}$ be the decision tree built by Algorithm \ref{alg:randLearner}. Then,
$$\sum_{v \in {\cal D}} t_v ||diam(v)||_1 \le d \sum_{i=1}^d  UB_i({\cal D}_i).$$
\label{lem:tree-combination}
\end{lemma}
\begin{proof}
For a node $j \in {\cal D}_i$,
let $S_{i,j}$ be the (possibly empty) set of nodes in the tree
${\cal D}$ that correspond to  $j$, that is, the nodes that use the cut associated with the node $j$ from 
${\cal D}_i$. 
We have
\begin{equation}
\sum_{v \in {\cal D}} t_v ||diam(v)||_1 =\sum_{i=1}^d  \sum_{j \in {\cal D}_i}
\sum_{u \in S_{i,j}}  t_u || diam(u)||_1.
\label{eq:24nov-UB}
\end{equation}

Moreover, we have that
\begin{eqnarray}
 \sum_{u \in S_{i,j}} t_u  ||diam(u)||_1 \le \sum_{u \in S_{i,j}}  t_u \cdot d  \cdot diam(u)_i  \le \label{eq:bound1}\\
  \sum_{u \in S_{i,j}} d \cdot t_u \cdot \max_{u \in S_{i,j}}\{ diam(u)_i\} \le \sum_{u \in S_{i,j}} d \cdot t_u \cdot diam(j)_i,  \label{eq:bound2}\end{eqnarray}
where the first  inequality  in (\ref{eq:bound1}) holds because $i$ is the coordinate
for which the diameter of $u$ is maximum and the 
inequality (\ref{eq:bound2}) holds
 because
the set of centers in $u$ is a subset of the set of centers that reach the node identified by $j$ in
${\cal D}_i$.

\begin{claim}
For a node $u \in S_{i,j}$, let ${\cal X}_u \subseteq {\cal X}$ be the set of points that reach $u$
in ${\cal D}$.
Then,  ${\cal X}_u \cap {\cal X}_{u'} = \emptyset $ for every $u,u'\in S_{i,j}$,
with $u \ne u'$. 
\end{claim}
\begin{proof}
Let $w$ be the least common ancestor of $u'$ and $u$ in ${\cal D}$.
If $w \notin \{ u,u'\}$  then the cut associated with $w$ splits
${\cal X}_w$ into two disjoint regions, one of them containing 
${\cal X}_u$ and the other containing 
${\cal X}_{u'}$ so that ${\cal X}_u$ and $ {\cal X}_{u'}$ are disjoint.

If $w \in \{ u,u'\}$ let us assume w.l.o.g. that $w=u$. In this case,
the cut $(i,\theta^j)$, associated with $u$, splits ${\cal X}_u$ into two regions, one
of them containing all the reference centers that reach $u'$. These  centers
are contained in the set of reference centers of one of the children of $j$ in 
${\cal D}_i$ and, hence, the LCA in ${\cal D}_i$ of the set of centers that reach
$u'$ is not $j$, that is, $u' \notin S_{i,j}$. This contradiction  shows that this case cannot occur.
\end{proof}

From  the previous claim we get that 
$$ 
 \sum_{u \in S_{i,j}} t_u  \le T_j.
 $$
It follows from (\ref{eq:bound1})-(\ref{eq:bound2}) and the above inequality that
$$ \sum_{u \in S_{i,j}} t_u  ||diam(u)||_1 \le  d \cdot T_j \cdot diam(j)_i. $$ 

Hence, it follows from (\ref{eq:24nov-UB}) that 
\begin{gather*}
    \sum_{v \in {\cal D}} t_v ||diam(v)||_1 \le \sum_{i=1}^d  \sum_{j \in {\cal D}_i}
d \cdot T_j \cdot diam(j)_i= \\
d \sum_{i=1}^d  UB_i({\cal D}_i ). \qedhere
\end{gather*}
\end{proof}



 
\fi

\section{The $k$-means problem}

\subsection{Improved bounds for low dimensions} 

The result we obtained  for  the $k$-medians problem
can be extended to the $k$-means problem:

\begin{thm}
The price of explainability for $k$-means is $O(d k \log k)$.
\label{thm:k-means}
\end{thm}

From an  algorithmic perspective, in order to establish the theorem, we only need to replace the definition of 
$UB_i({\cal D}')$ for a tree ${\cal D}'$ in ${\cal F}_{a,b}$
with
$$UB'_i({\cal D'})=\sum_{j=a}^{b-1} T_j \cdot (diam(j)_i)^2.$$

Note that the only difference is the replacement of 
$ diam(j)_i$ with  $(diam(j)_i)^2$.
As a consequence, for the $k$-means problem, the tree ${\cal D}_i$ is 
defined as the tree ${\cal D}'$ in ${\cal F}_{1,k}$ for which 
$ UB'_i({\cal D}')$ is minimum. It can  also be constructed via dynamic programming.
Theorem \ref{thm:k-means} can be proved 
 by using arguments similar to those employed to bound
 the price of explainability for $k$-medians. 
 \ifICML 
 \else
 The following inequalities are, respectively, counterparts of the 
inequalities (\ref{eq:23NovUB}), (\ref{eq:23NovUB2}), (\ref{eq:nov19-3}) and (\ref{eq:nov19}):

\begin{equation}
cost({\cal C}) \le OPT + \sum_{v \in {\cal D}} t_v ||diam(v)||^2_2, 
\label{eq:k-means-bound-0}
\end{equation}

\begin{equation}
\label{eq:k-means-bound-1}
UB'_i({\cal D}_i) \le 2 k \log k \left ( \sum_{j=1}^{k-1} (c_i^{j+1}-c^j_i)^2 \cdot  T_j \right),
\end{equation}



\begin{equation}
\label{eq:k-means-bound-2}
 \left ( \sum_{j=1}^{k-1} (c_i^{j+1}-c^j_i)^2 \cdot T_j \right)/2
  \le  OPT_i,
 \end{equation}
%

\begin{equation}
\label{eq:nov19-kmeans}
\sum_{v \in {\cal D}} t_v ||diam(v)||^2_2 \le d \sum_{i=1}^d UB'_i ({\cal D}_i).
\end{equation}

From the three last  inequalities  and  the identity  $OPT=\sum_{i=1}^d OPT_i$,
we obtain 
$$\sum_{v \in {\cal D}} t_v ||diam(v)||^2_2 \le  4 d k \log k \cdot OPT.$$
This together with the inequality
(\ref{eq:k-means-bound-0}) allows us to establish Theorem \ref{thm:k-means}.

Inequality (\ref{eq:k-means-bound-0}) is proved in \cite{dasgupta2020explainable}.
The validity  of inequalites  
(\ref{eq:k-means-bound-2}) and 
(\ref{eq:nov19-kmeans}) can be established  by using exactly the same arguments employed
to prove their counterparts. More specifically,
the proof of Lemma \ref{lem:k-means-lower-bound} can be used  for 
 the former while the proof of Lemma \ref{lem:tree-combination} can be used for the latter.

\remove{
Inequality (\ref{eq:k-means-bound-1}) incurs an
extra factor of $k$ w.r.t. its counterpart.
This can be proved via a simple adaptation of the proof employed to establish  Lemma \ref{lem:UB-k-medians-1},
in particular Equation (\ref{eq:23Nov-expr}).
Lemma \ref{} in Appendix shows the correctness of inequality  (\ref{eq:k-means-bound-1}) .
}

The inequality (\ref{eq:k-means-bound-1}) incurs an
extra factor of $k$ with respect to  its counterpart.
In order to prove  this inequality,  we apply the arguments of  the proof of Lemma \ref{lem:UB-k-medians-1}.
The only required adaptation  consists of  replacing 
Equation (\ref{eq:23Nov-expr})
with the inequality
\begin{equation}
 UB'_i({\cal D}_i) \le k \sum_{j=1}^{k-1} T_j \cdot  \sum_{\ell=r(j)}^{s(j)-1} (c^{\ell+1}_i -c^{\ell}_i)^2.
\label{eq:21dc}
\end{equation}
Inequality (\ref{eq:21dc})  holds because 
$$ UB'_i({\cal D}_i) = \sum_{j=1}^{k-1} T_j  \cdot ( c_i^{s(j)}-c_i^{r(j)})^2, $$ and a simple application
of  
Jensen's inequality assures that 
$$(c_i^{s(j)}-c_i^{r(j)})^2 \le k \sum_{\ell=r(j)}^{s(j)-1} (c^{\ell+1}_i -c^{\ell}_i)^2.$$

\fi

\subsection{A practical algorithm} 
We propose  a simple greedy algorithm, denoted by {\tt Ex-Greedy}, for building explainable clustering for the $k$-means problem.
We provide  evidence that it 
 performs very well in practice.

The algorithm starts with the set $S$ of representatives of an unrestricted $k$-clustering ${\cal C}^{ini}$
for the dataset ${\cal X}$ and then  builds a decision tree ${\cal D}$ with $k$ leaves,
where each of them includes exactly one representative from $S$.

Let $u$ be a node of the decision tree and let ${\cal X}^u$ and ${\cal S}^u$ be,
respectively, the set of points and the set of reference centers (representatives of ${\cal C}^{ini}$)  that reach
$u$.  We define the cost of
a partition  $(L,R)$  of the points in ${\cal X}^u \cup {\cal S}^u$ as  


\begin{align*}
 cost(L,R)= & \sum_{\xvec \in L \cap {\cal X}^u } \min_{\cvec \in L \cap {\cal S}^u } || \xvec - \cvec ||_2^2 + \\ 
 & \sum_{\xvec \in R \cap {\cal X}^u} \min_{\cvec \in R \cap {\cal S}^u } || \xvec - \cvec ||_2^2. 
\end{align*}

\remove{
The algorithm selects at each node, including more than one reference center, the axis-aligned cut that induces a partition with minimum cost.
}

To split a node $u$, that is reached by more than one representative, {\tt Ex-Greedy}  selects 
 the axis-aligned cut that induces a partition with minimum cost.

\remove{
The next theorem gives an upper bound on the  cost of the explainable clustering obtained by {\tt Ex-Greedy}. 

\begin{thm}
If {\tt Ex-Greedy} starts with an unrestricted $k$-clustering ${\cal C}^{ini}$
then it obtains a clustering ${\cal C}$ that satisfies 
$$cost({\cal C}) \le O(k^2) cost({\cal C}^{ini}).$$
\end{thm}

This result can be established by using the same arguments
employed to prove that the bound of IMM. 
}

{\tt Ex-Greedy} can be implemented in  $O( n d  k H + n d \log n)$ time,
where $H$ is the depth of the resulting decision tree.
Note that $H \le k$ and in many relevant applications $k$ is small.
The time complexity corresponds to $H$ iterations of Lloyd's $k$-means algorithm.
\ifICML 
\else

\subsubsection{An efficient implementation}
To achieve this time complexity, in the preprocessing phase,
{\tt Ex-Greedy} builds the following data structures: 

\begin{itemize}

\item a list ${\tt SL_i}$, for each $i \in [d]$, containing the points in ${\cal X} \cup S$
sorted by coordinate $i$; 
\item  a list  ${\tt M}_{\xvec}$ of size $k$, for each $\xvec \in {\cal X}$, that stores
the $k$ centers sorted  by increasing order of their distances to $\xvec$.
\end{itemize}

The lists ${\tt SL_i}$ can be built in $O( d n \log n )$ time and the lists 
${\tt M}_{\xvec}$  in $O(n k \log k )$ time. 

\medskip

To decide how to split  the root  the algorithm finds  the partition with minimum cost
for each coordinate $i \in [d]$ and then selects the one with minimum cost among them.

Fix $i \in [d]$. The algorithm scans the list  ${\tt SL_i}$ from left to the right and 
evaluates the cost of  $n-k+1$ partitions where the $j$-th
one, namely $(L_j,R_j)$,  separates the first $j$ points in 
${\tt SL_i}$ from the remaining ones. During the scan
the algorithm makes use of two vectors of size $n$, ${\tt V}_L$ and ${\tt V}_R$.
Right after evaluating $(L_j,R_j)$,  ${\tt V}_L$ (resp. ${\tt V}_R$) stores,
for each $\xvec$ that lies at $L_j$ (resp. $R_j$),
the center that is closest to $\xvec$ among those that also lie
in   $L_j$ (resp. $R_j$).
The only difference is that 
${\tt V}_L[\xvec]$  stores the center directly while
${\tt V}_R[\xvec]$   stores the position of the center  in  ${\tt M}_{\xvec}$.

Let us consider the  moment in which  the algorithm has just calculated  the cost
${\tt Cost_{j}}$ of the $j$th partition $(L_j,R_j)$.
  To obtain ${\tt Cost_{j+1}}$ and update ${\tt V}_L$ and ${\tt V}_R$,   the algorithm
  first set ${\tt Cost_{j+1}}={\tt Cost_{j}}$ and then proceeds according to the following  cases:

\medskip

{\it Case 1}. The $(j+1)$-th point in  ${\tt SL_i}$ corresponds to a point $\xvec$ in ${\cal X}$. Then, the algorithm evaluates ${\tt Cost_{j+1}}$  in $O(k)$ time as follows: 
\begin{itemize}
\item[i] it obtains the center $\cvec_R$ in  $R_j$ that is closest to $\xvec$. This is
done in $O(1)$ time since ${\tt V}_R[\xvec]$ points to this center; 
\item[ii] By scanning
${\tt M_{\xvec}}$ it obtains the center $\cvec_L$ in $L_{j+1}$ that is closest
to $\xvec$ and then 
 updates ${\tt V}_L[\xvec]$  to  $\cvec_L$ . This requires  $O(k)$ time
\item[iii] it updates  ${\tt Cost_{j+1}}$ to  ${\tt Cost_{j}}  +  || \xvec  - \cvec_L ||_2^2 -  || \xvec-\cvec_R||_2^2$.
\end{itemize}

\medskip

{\it Case 2}.  The $(j+1)$-th point in   ${\tt SL_i}$ corresponds to a reference center $\cvec$ in $S$.Then, the algorithm evaluates ${\tt Cost_{j+1}}$  in  $O(n)$ amortized time as follows: 

\begin{itemize}
\item[i]  
 for each point $\xvec$  in $L_j$, {\tt Ex-Greedy} 
compares $||\cvec-\xvec||_2^2$ with
$|| {\tt V}_L[x] -\xvec||_2^2$.
If $\cvec$ is the  closest then 
it updates ${\tt Cost_{j+1}}$ to ${\tt Cost_{j+1}}+ || \xvec-\cvec||_2^2 -
|| \xvec-  {\tt V}_L[x] ||_2^2$ and ${\tt V}_L[ \xvec]$ to $\cvec$.
This requires $O(n)$ time.

\item[ii] for each point $\xvec$  in $R_j$
it verifies whether $V_R[\xvec]$ points to $\cvec$.  In the negative case, nothing is done.
In the positive case,
it scans  ${\tt M}_\xvec$, starting from  $V_R[\xvec]$ towards to its end ,
until it finds a center $\cvec'$ that  lies in $R_j$. 
Then it updates ${\tt Cost_{j+1}}$ to
${\tt Cost_{j}}- || \xvec-\cvec||_2^2 + || \xvec-\cvec'||_2^2.$  
This operation  requires $O(n)$ amortized time since the total
cost spent on these scans, when we  take into account moving the $k$ centers, is
$O(nk)$. 

\end{itemize}

The algorithm applies the cut with minimum cost and then recurses on each the children of the root.
To process a child $u$ of the root, the implementation updates the data structures
${\tt SL_i}$ and ${\tt M_{\xvec}}$  to only comprise the
points and the reference center that reach $u$.
Each list ${\tt SL_i}$ can be updated in $O(n)$ time by removing the points
and the reference centers that do not reach $u$.
Similarly, each list  ${\tt M_{\xvec}}$ can be updated in $O(k)$ time by removing points and the reference centers that do not reach $u$.

\subsubsection{Experiments}

\fi

\ifICML
\else

\cite{dasguptaexplainable-workshop,frost2020exkmc} compared  6 methods that build explainable clusterings,  over 10 datasets. These methods also allow the construction of decision trees with more than $k$ leaves but this is not
relevant for our experiments.
For trees with $k$ leaves, the {\tt IMM} algorithm proposed in \cite{dasgupta2020explainable} obtained the best results, or was very close to it, 
for all datasets but one ({\tt CIFAR-10}).

\begin{table}[]
\caption{Comparison of {\tt Ex-Greedy} and {\tt IMM} over 10 datasets}
\label{tbl:experiments}
\begin{center}
\begin{tabular}{c|c|c|c|c|c}
Dataset 	 & n & d & k & {\tt IMM} &	{\tt Ex-Greedy} \\  \hline 
{\tt BreastCancer} & 569 & 30 & 2 & 1.00 & 1.00 \\
{\tt Iris}        & 150 & 4 & 3  & 1.04 & 1.04 \\
{\tt Wine}      & 178 & 13 &  3  & 1.00 & 1.00 \\
{\tt Covtype}    & 581,012 & 54 & 7   & 1.03  & 1.03  \\
{\tt Mice}        & 552 & 77 &  8  & 1.12 & {\bf 1.09} \\
{\tt Digits }     & 1,797 & 64 & 10 &  1.23 & {\bf 1.21} \\
{\tt CIFAR-10}       & 50,000 & 3,072 & 10 &  1.23    & {\bf 1.17} \\
{\tt Anuran}      & 7,195 & 22 & 10  & 1.30     & {\bf 1.15}     \\
{\tt Avila}         & 20,867 & 12 & 12 & 1.1 & {\bf 1.09} \\
{\tt Newsgroups} & 18,846 & 1,069 & 20 &  1.01     & 1.01     \\
\end{tabular}
\end{center}
\end{table}

\fi

%

Given the success of {\tt IMM}, we compared it with our method  {\tt Ex-Greedy}  on the same datasets. The column {\tt IMM} (resp. {\tt Ex-Greedy}) of Table \ref{tbl:experiments} shows the average ratio between the cost of the clustering
obtained by   {\tt IMM} (resp.  {\tt Ex-Greedy})   and that
 of the initial unrestricted clustering ${\cal C}^{ini}$ produced by scikit-learn's {\tt KMeans} algorithm \cite{scikit-learn}.
 \ifICML
\else
Following \cite{frost2020exkmc}, the value of  $k$ is the number of classes for the classification task 
associated with the dataset.

   Each dataset was run for 10 iterations, with random seeds from 1 to 10, to ensure the reproducibility of results. For each iteration, we initially achieve an unrestricted solution ${\cal C}^{ini}$ by running the {\tt KMeans} algorithm provided in the {\tt scikit-klearn} package with default parameters. We then pass ${\cal C}^{ini}$ to the implementation of {\tt IMM} from \cite{frost2020exkmc}, available at \url{https://github.com/navefr/ExKMC}, and to our implementation of {\tt Ex-Greedy}, to find two explainable clustering solutions induced by decision trees.

\fi
 
\ifICML
\else

For 5  datasets, the results were very similar while 
for the others (bold in Table \ref{tbl:experiments})   {\tt Ex-Greedy} performed  better than {\tt IMM}.
Figure \ref{fig:boxplot} presents box plots  for the 5  datasets where there was a difference
of at least 0.01 on the average results.  It is interesting to note that 
the dispersion of {\tt Ex-Greedy} is considerably smaller.

In terms of  running time both methods spent less than 1 second, for 6 datasets.
For the remaining datasets {\tt IMM} was the  fastest as shown in  Table
\ref{tbl:experiments-runtime}. In spite of that, we understand that {\tt Ex-Greedy}
is fast enough to be used in practice.



\begin{center}
\begin{figure}
\caption{Box Plots for the datasets with  difference at least 0.01 }
\label{fig:boxplot}

		 \includegraphics[scale=0.45]{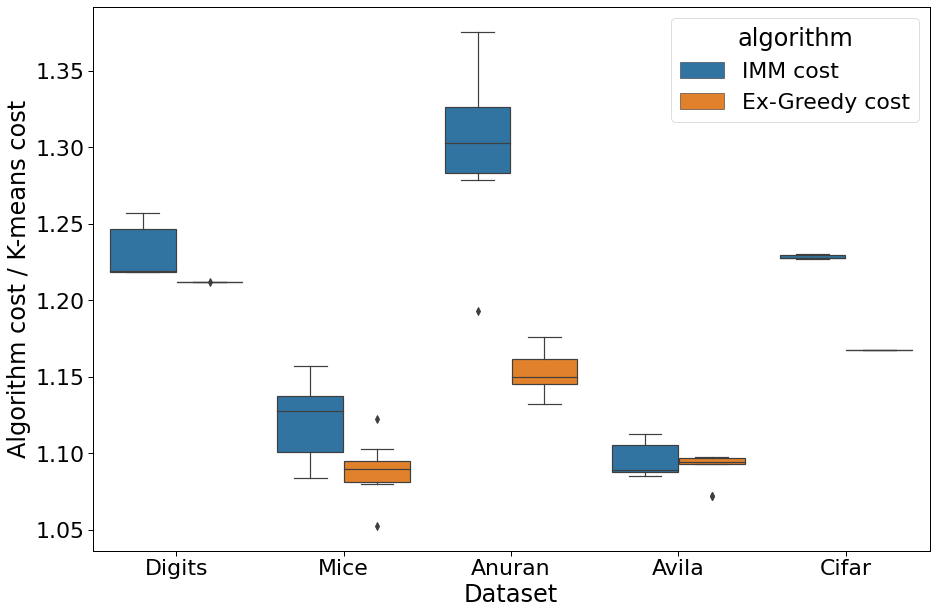}
\end{figure}
\end{center}

\begin{table}[]
\caption{Average running times in seconds for  {\tt Ex-Greedy} and {\tt IMM}}
\label{tbl:experiments-runtime}
\begin{center}
\begin{tabular}{c|c|c}
Dataset 	 & {\tt IMM} (sec)&	{\tt Ex-Greedy} (sec)\\  \hline 
{\tt Avila}      &  1.7 & 2.4 \\
{\tt Covtype}      & 42  & 53  \\
{\tt Newsgroups}   & 41     & 102     \\
{\tt CIFAR-10}    & 312    & 378 
\end{tabular}
\end{center}
\end{table}

\subsubsection{Details of the experimental settings and the datasets}
All our experiments were executed in a  MacBook Air, 8Gb of RAM, processor 1,6 GHz Dual-Core Intel Core i5, 
executing   macOS Catalina, version 10.15.7.
Our code is availble in \url{https://github.com/lmurtinho/ExKMC.}

The datasets {\tt Iris}, {\tt  Wine}, {\tt Breast Cancer}, {\tt Digits}, {\tt Covtype},
{\tt Mice} and {\tt Newsgroup}
 are available in  Python's {\tt scikit-learn};
 {\tt Cifar-10} is available in {\tt TensorFlow};
{\tt Anuran} and {\tt Avila} were downloaded from UCI.

For {\tt Mice},  the examples with missing
values were removed.
For {\tt Avila}, the training set and the testing set are used together.
Finally, for {\tt Newsgroup}, we removed headers, footers, quotes, stopwords, and words that
either appear in less than $1\%$ or more than  
 $10\%$ of the documents, following \cite{frost2020exkmc}.

\fi

\section{Maximum-Spacing Clustering}
We show that the price of explainability  
for the  maximum-spacing  problem is $\Theta(n-k)$.

\subsection{Lower bound}
The following simple construction shows that 
the price of explainability  is $ \Omega(  n-k) $.

Let $C_1=\{(0,i)| 0 \le i \le p\}  \cup
 \{(i,0)| 0 \le i \le p\}$.
Moreover,  for $i=2,\ldots,k$,
 let $C_i=\{(i-1)(p-1),(p-1)\}$.
 The dataset ${\cal X}$ for our instance is given by $C_1 \cup \ldots \cup C_k$.



The unrestricted $k$-clustering $(C_1,\ldots,C_k)$ has spacing $p-1=(n-k)/2-1$.
On the other hand, every explainable $k$-clustering has
spacing 1. To  see that,  note that we cannot have all the points of $C_1 \cup C_2$ in the same
cluster, for otherwise we would have at most $k-1$ clusters.
Thus, we need to separate at least 2 points from $C_1 \cup C_2$ and the only way to accomplish that,
via axis-aligned  cuts, forces  the  separation of 2  points in $C_1$ that are at distance 1 from each other. Thus,
 the spacing will be 1.  
 
\begin{lemma}
The price of explainability for the maximum-spacing clustering problem    is $\Omega(n-k)$.
\end{lemma}


\subsection{Upper Bound}
We present an algorithm that always  obtains 
an explainable clustering with spacing $O((n-k) OPT)$, 
where $OPT$ is the spacing of the optimal unrestricted clustering.
That, together with the previous lemma, implies that the price of 
explainability  for the maximum-spacing problem is $\Theta(n-k)$.

Algorithm \ref{alg:single-link} receives  an optimal $k$-clustering ${\cal C}^* $ as input and 
uses it as  a guide to transforming an initial single cluster containing  all points of ${\cal X}$
into an explainable $k$-clustering.
The existence of cluster $C$ at line (*) follows from a simple pigeonhole argument.
The motivation for this choice is that $C$ has two points at distance at least $OPT$,
which is used to show the existence of a cut with a large enough margin.



\begin{algorithm}[H]
  \caption{{\tt Ex-SingleLink}(${\cal X}$)}
  \begin{algorithmic}[]
  	\small
		
		\STATE ${\cal C}^* \leftarrow$  optimal unrestriced $k$-clustering for points in ${\cal X}$.

		\STATE ${\cal C} \leftarrow $ single cluster containing all points of ${\cal X}$

\FOR{$i=1,\ldots,k-1$}

\STATE Select a cluster $C \in {\cal C}$ that contains two points that lie in
different clusters in ${\cal C}^*$. \,\, (*)

\STATE Split $C$  using an axis-aligned cut that
yields  a 2-clustering $(C',C'')$ with maximum possible spacing.

\STATE Remove  $C$ from ${\cal C}$ and update ${\cal C}$ to 
${\cal C}  \cup \{C',C''\}$


\ENDFOR

  \end{algorithmic}
  \label{alg:single-link}
\end{algorithm}

\begin{lemma} 
Given a set of points ${\cal X}$, {\tt Ex-SingleLink}(${\cal X}$)  obtains a $k$-clustering ${\cal C}$ with  spacing at least $OPT/(n-k)$,
where $OPT$ is the  spacing of an optimal unrestricted clustering.
\end{lemma}  


\ifICML

\else

\begin{proof}

First, we observe that it is always possible to properly execute line (*) of 
 {\tt Ex-SingleLink}.
In fact,  if we pick $k$ points covering all the $k$ clusters of  ${\cal C^*}$ then,
by the pigeonhole principle, two of them will lie in the same group in ${\cal C}$ since
${\cal C}$ has less than $k$ groups when line (*) is executed.

To establish the result it suffices to prove  that there is always an axis-aligned  cut that splits
the selected cluster  $C$
into two clusters with spacing at least $OPT/(n-k)$.

Let $\pvec$ and $\qvec$ be two points  in $C$  that lie in distinct  clusters
in ${\cal C}^*$ and 
let $G=(V,E)$ be a graph, where $V$  is the set of points in $C$ and $E$
connects points in $V$ with distance smaller than $OPT/(n-k)$.   
Moreover, let $F=(T_1,\ldots,T_\ell)$ be a forest that is obtained by  running Kruskal's MST algorithm
on  
$G$.

\begin{claim}
Points in $C$ that belong to distinct  clusters of ${\cal C}^*$ must
also belong to different trees in forest $F$. 
\label{clm:04Jan}
\end{claim}
\begin{proof}
For the sake of contradiction we assume that  the claim does not hold.
In this case, there would be a  path from $\pvec$ to $\qvec$
in  $F$ and this path  would have an edge joining two points that belong to different clusters in ${\cal C}^*$,
which  cannot occur since their distance is at least $OPT > OPT/(n-k)$.
\end{proof}

The previous claim implies that $\ell \ge 2$ since $\pvec$ and $\qvec$ belong to different clusters.
We say that   an axis-aligned cut is {\em good} with respect to a cluster $C$ if it satisfies the following properties:
(i) it separates the points in $C$ into  two non-empty clusters and (ii) it  
 does not separate points that lie in the same tree of $F$.
If a good cut exists,  then we can use it to split $C$ into
two clusters with spacing at least $OPT/(n-k)$ since, by construction, points in different trees have distance at least $OPT/(n-k)$. For the sake of contradiction let us assume that  such a cut does not exist.

For each $j \in [d]$ let  $I^{\pvec \qvec}_j$ be the real interval
that starts in $\min\{p_j,q_j\}$ and ends in
$\max\{p_j,q_j\}$, that is,
 $I^{\pvec \qvec}_j=[ \min\{p_j,q_j\}, \max\{p_j,q_j\}]$.


Moreover, for each tree $T$ in $F$,
let  $I^T_j$ be the interval that starts at $\min \{ x_j | \xvec \mbox{ is a node in } T\}$ 
and ends at $\max \{ x_j | \xvec \mbox{ is a node in } T\}$. Finally,
 for each edge $e=uv$ in $F$   and each 
$j \in [d]$, let $I^e_j$ be the real interval that
starts at $\min \{u_j,v_j\}$ and ends at 
$\max \{u_j,v_j\}$.
For a real interval $I$, let $len(I)$ be its length.

Since there are no good cuts,  for $j=1,\ldots,d $, 
we have
$$  \sum_{T \in F} len(I^{T}_j)
 \ge len(I^{\pvec \qvec}_j).$$  
From the triangle inequality 
we obtain
$$ \sum_{T \in F} \sum_{e \in T} len(I^e_j) \ge \sum_{T \in F} len(I^{T}_j).$$
From the two previous inequalities we get 
$$ \sum_{e \in F} len(I^e_j) = \sum_{T \in F} \sum_{e \in T} len(I^e_j) \ge len(I^{\pvec \qvec}_j).$$
A simple application of Jensen  inequality shows that
$$ \sum_{e \in F}  len(I^e_j)^2  \ge \frac{(len(I^{\pvec \qvec}_j))^2}{f},$$
where $f$ is the number of edges in $F$.
By adding the above inequality for all $j \in [d]$ we get 
$$ \sum_{e \in F} ||e||^2 _2 \ge \frac{1}{f} ||\pvec-\qvec||_2^2 \ge \frac{OPT^2}{f}, $$
where $||e||_2$ is the distance between the two endpoints of edge $e$.

The last inequality implies $ ||e||_2 \geq OPT/f,$
 for some edge $e$.  
Thus, to obtain a contradiction, it suffices to show that $f \le n-k$, since
we cannot have edges in $F$ with distance $\ge OPT/(n-k)$

 To see that  $f \le n-k$,
 let $k'$ be the number of clusters in ${\cal C}$ 
 that are singletons and let $S'$ be the set of points in these clusters.
Moreover, let $S \subseteq {\cal X} -S '$  be a set of $k-k'$ points
with each of them belonging to a different cluster    in ${\cal C}^*$.
Note that cluster  $C$  is not a singleton since $\pvec,\qvec \in C$. 
Since both $C$ and $S$ are  subsets of ${\cal X}-S'$ we have 
 $|C \cup S| = |C| +|S| - |C \cap S| \le n -k'$ so that
 $|C| - |C \cap S| \le n-k$.
It follows from Claim \ref{clm:04Jan} that
the number of trees in $F$ is at least $|C \cap S|$ and, as a result, its number of edges $f$ satisfies
 $f \le |C| -|C \cap S| -1 < n- k$ edges.
\end{proof}
\fi

We can state the main result of this section.

\begin{thm} 
The price of explainability for the  maximum-spacing problem is
$\Theta(n-k)$.
\end{thm}

\ifICML
\else
	\bibliographystyle{siam}

\fi

\bibliography{biblio.bib}

\end{document}